\documentclass{article} 
\usepackage[table]{xcolor}
\usepackage{nips15submit_e,times}
\usepackage{hyperref}
\usepackage{url}
\usepackage[numbers]{natbib}

\usepackage{amssymb}
\usepackage{latexsym}
\usepackage{amsbsy}
\usepackage{amsmath}
\usepackage{amsthm}
\usepackage{enumerate}
\usepackage{xspace}
\usepackage{bm}
\usepackage{booktabs}
\usepackage{mathtools}
\usepackage{graphicx}
\usepackage{subfigure}
\usepackage{algorithm}
\usepackage{algorithmicx}
\usepackage[noend]{algpseudocode}
\usepackage{stmaryrd}
\usepackage{multirow}

\newtheorem{lemma}{Lemma}
\newtheorem{theorem}{Theorem}

\usepackage{mario-def}

\definecolor{mydarkblue}{rgb}{0,0.08,0.45}
\hypersetup{colorlinks=true, linkcolor=mydarkblue, citecolor=mydarkblue, filecolor=mydarkblue, urlcolor=mydarkblue}

\title{Efficient Learning of Ensembles with QuadBoost}

\author{
Louis Fortier-Dubois, Fran\c{c}ois Laviolette, Mario Marchand,\\ {\bf Louis-Emile Robitaille, and Jean-Francis Roy} \\
D\'{e}partement d'informatique et de g\'{e}nie logiciel\\
Universit\'{e} Laval\\
Qu\'{e}bec, Canada, G1V 0A6 \\
\texttt{firstname.lastname@ift.ulaval.ca} \\
}

\newcommand{\ahpair}{{\alb,\hb}}

\nipsfinalcopy 

\begin{document}

\maketitle

\begin{abstract}   
We first present a general risk bound for ensembles that depends on the $L_p$ norm of the weighted combination of voters which can be selected from a continuous set. We then propose a boosting method, called QuadBoost, which is strongly supported by the general risk bound and has very simple rules for assigning the voters' weights. Moreover, QuadBoost exhibits a rate of decrease of its empirical error which is slightly faster than the one achieved by AdaBoost. The experimental results confirm the expectation of the theory that QuadBoost is a very efficient method for learning ensembles. 
\end{abstract}

\section{Introduction}

As data is becoming very abundant, machine learning is now confronted with the challenge of having to learn complex models from huge data sets. Among the learning algorithms which seem most likely to be able to scale up to meet this challenge are ensemble methods based on the idea of boosting weak learners~\citep{s-90}. Take AdaBoost~\citep{fs-97} for example. If a weak learner is (almost always) able to produce in linear time a classifier achieving an empirical error just slightly better than random guessing, then the exponential rate of decrease of the training error of AdaBoost will give us a good majority vote in linear time.  

After AdaBoost was published, it soon became clear that infinitely many surrogate loss functions~\citep{mbbf-00} and regularizers could be used for boosting and, without surprise, many variants have been proposed---to the point where the practitioner is often completely overwhelmed when confronted with the choice of picking a boosting algorithm for his learning task. Are some algorithms better than others? If so, then under what circumstances are they better? If not, then are they, somehow, all equivalent? In an attempt to answer these questions we have decided to search for a risk bound guarantee that applies to all ensemble methods, no matter what are the surrogate loss and regularizer used by the algorithm. What comes out from the risk bound presented in the next section is the distinct difference between a $L_1$ norm regularizer and all the other $L_p$ norm regularizers with $p > 1$. This difference appears to be fundamental in the sense that the Rademacher complexity of a unit $L_{p>1}$ norm combination of functions depends explicitly on the number of functions used in the ensemble while no such dependence occurs for the $L_1$ norm case. Consequently, an explicit control of the number of voters in the ensemble should be exercised while boosting with a $L_{p>1}$ regularizer, but no such control is needed while regularizing with the $L_1$ norm.   

Concerning the issue of the surrogate loss to be used for boosting, we propose the simple quadratic loss (hence the name QuadBoost). Although the theory suggests using the hinge loss, this leads to linear programming algorithms which could become computationally prohibitive with large number of voters and huge data sets. The quadratic loss, on the other hand, leads to very simple rules for setting the weights on the voters, does not need to assign weights on the training examples, and exhibits a rate of decrease of the training error which is slightly faster than the one achieved by AdaBoost. The use of the quadratic loss for boosting was already proposed by~\citet{by-03} and its effectiveness has been analysed through the bias-variance decomposition. Here, instead, we analyse it through a risk loss bound and propose several regularized variants that were not considered by~\citet{by-03}. The experimental results confirm the expectation of the theory that QuadBoost is a very efficient method for learning ensembles and, consequently, is likely to be effective for learning complex models from data.  

\section{A General Risk Bound for Ensembles}
\label{sect:RBE}

We consider the difficult task of finding classifiers having small expected zero-one loss. In the supervised learning setting, the learner has access to a training set $S\eqdef\{(x_1, y_1),\ldots,(x_m,y_m)\}$ of $m$ examples where each example $(x_i,y_i)$ is drawn independently from a fixed, but unknown, distribution $D$ on $\Xcal\times\Ycal$. For the binary classification case, the input space $\Xcal$ is arbitrary, whereas the output space $\Ycal = \{-1, +1\}$. Given access to $S$, the task of the learner is to find, in reasonable time, a classifier $f: \Xcal\ra\Ycal$ having a small expected zero-one loss $\Eb_{(x,y)\sim D} I(f(x)\ne y)$, where $I(a) = 1$ if predicate $a$ is true, and $0$ otherwise.

We are not only concerned here with the problem of finding classifiers with good generalization (\ie, small expected zero-one loss), but also with the running time complexity of finding such classifiers. In that respect, ensemble methods, such as AdaBoost~\citep{fs-97}, appear to us as mostly promising. Let us then investigate these methods with respect to both objectives. 

As is often the case with ensemble methods, we assume that we have access to a (possibly continuous) set $\Hcal$ of real-valued functions that we call the set of possible \emph{voters}. Our task is to select from $\Hcal$ a finite subset of $n$ voters
on which a weighted majority vote classifier is produced.
Let $\hb\eqdef (h_1, \ldots, h_n)$ denote the vector formed by concatenating these $n$ voters and let $\alb\eqdef (\al_1,\ldots,\al_n)$ denote the vector of $n$ real-valued weights used to weight the voters. For any input $x\in\Xcal$, the output $f_\ahpair(x)$ on $x$ of the weighted majority vote is given by
$$
f_\ahpair(x)\ =\ \sgn\LP \alb\cdot\hb(x)\RP\ \eqdef\ \sgn\LP\sum_{i=1}^n \al_i h_i(x)\RP\, ,
$$
where $\sgn(z)=+1$ if $z>0$, and $-1$ otherwise. 


To find out what the majority vote $f_\ahpair$ should optimize on the training data $S$ to have good generalization, we have investigated guarantees known as uniform risk bounds. In particular, those which are based on the Rademacher complexity are particularly appealing and tight. In a nutshell, the Rademacher complexity of a class of functions measures its capacity to fit random noise. More precisely, given a set $S$ of $m$ examples, the \emph{empirical Rademacher complexity} $\Rcal_S(\Fcal)$ of a class $\Fcal$ of real-valued functions and its expectation $\Rcal_m(\Fcal)$ are defined as
$$
\Rcal_S(\Fcal)\ \eqdef\ \esp{\sgb} \sup_{f\in\Fcal} \frac{1}{m} \sum_{i=1}^m
\sg_i f(x_i) \quad;\quad \Rcal_m(\Fcal)\ \eqdef\ \esp{S\sim D^m}\Rcal_S(\Fcal)\, ,
$$
where $\sgb\eqdef(\sg_1,\ldots,\sg_m)$ and where each $\sg_i$ is a $\pm 1$-valued random variable drawn independently according to the uniform distribution.

Given a weighted majority vote $f_\ahpair$ of functions taken from a function class $\Hcal$, let $\|\alb\|_p$ denote the $L_p$ norm of vector $\alb$ for any $p\ge 1$ and let $\dim(\alb)$ denote the dimension (\ie, the number of components) of vector $\alb$.  
An important issue concerning majority votes is the complexity of the set of functions induced by taking weighted combinations of functions at fixed norm. Hence, given a class $\Hcal$ of real-valued functions, let us consider 
$$
\Ccal_p^n(\Hcal)\ \eqdef\ \LC x\mapsto \alb\cdot\hb(x) \mid h_i\in\Hcal\ \forall i, \dim(\alb) = n, \|\alb\|_p = 1\RC\, .
$$
We already know that $\Rcal_S(\Ccal_1^n(\Hcal)) = \Rcal_S(\Hcal)$ for any $n$.
But what happens if the weighted combination is at unit $L_p$ norm for $p>1$? The next lemma, which is apparently new, tells us that taking a weighted combination of functions strictly increases the Rademacher complexity for $p>1$. 

We will assume, for the rest of the paper, that the class $\Hcal$ of voters is \emph{symmetric}. This means that if voter $h$ is in $\Hcal$, then voter $-h$ is also in $\Hcal$.  

\begin{lemma}\label{lem:holder}
For any symmetric class $\Hcal$ of real-valued functions, any $n\in\Naturals$, and any $p\in[1,+\infty)$, we have
$$
\Rcal_S(\Ccal_p^n(\Hcal))\ =\ n^{1-\frac{1}{p}}\Rcal_S(\Hcal)\, .
$$
\end{lemma}
\begin{proof}
Let $(1/q)\eqdef 1 - (1/p)$. We will make use of the known fact that H\"{o}lder's inequality is attained at the supremum, namely for all $p\ge 1$ and any vector $\vb$, we have
$$
\sup_{\alb:\|\alb\|_p=1} \sum_{i=1}^n \al_i v_i\ =\ \LP\sum_{i=1}^n |v_i|^q\RP^{1/q}\, .
$$
Consequently, we have
\begin{eqnarray*}
\Rcal_S(\Ccal_p^n) &=& \esp{\sgb}\sup_{h_1,\ldots,h_n}\sup_{\alb:\|\alb\|_p=1}\frac{1}{m}\sum_{k=1}^m\sg_k\sum_{i=1}^n \al_ih_i(x_k)\\
&=& \esp{\sgb}\sup_{h_1,\ldots,h_n}\sup_{\alb:\|\alb\|_p=1}\sum_{i=1}^n \al_i\LB\frac{1}{m}\sum_{k=1}^m\sg_kh_i(x_k)\RB\\
&=& \esp{\sgb}\sup_{h_1,\ldots,h_n}\LP\sum_{i=1}^n \left|\frac{1}{m}\sum_{k=1}^m\sg_kh_i(x_k)\right|^q\RP^{1/q}\\
&=& \esp{\sgb}\LP\sup_{h_1,\ldots,h_n}\sum_{i=1}^n \left|\frac{1}{m}\sum_{k=1}^m\sg_kh_i(x_k)\right|^q\RP^{1/q}\\
&=& \esp{\sgb}\LP n\sup_{h\in\Hcal}\left|\frac{1}{m}\sum_{k=1}^m\sg_kh(x_k)\right|^q\RP^{1/q}\\
&=& \esp{\sgb}\LP n\left|\sup_{h\in\Hcal}\frac{1}{m}\sum_{k=1}^m\sg_kh(x_k)\right|^q\RP^{1/q}\quad \textrm{(since}\ \Hcal\ \textrm{is symmetric)}\\ 
&=& n^{1/q}\Rcal_S(\Hcal)\, ,
\end{eqnarray*}
which proves the lemma.
\end{proof}

The next theorem, which is built on Lemma~\ref{lem:holder}, constitutes the main theoretical result of the paper. It provides a uniform upper-bound on the expected zero-one loss of weighted majority votes $f_{\ahpair}$ in terms of their empirical risk (\ie, the expected loss estimated on the training data) measured with respect to any loss function $\Lcal$ which upper-bounds the zero-one loss. The upper-bound also depends on the \emph{Lipschitz property} of the \emph{clipped version} of $\Lcal$. To define these notions precisely, let $\Lcal(y\alb\cdot\hb(x))$ denote the loss incurred by $f_{\ahpair}$, as measured by $\Lcal$, on example $(x,y)$. Then the loss incurred by $f_{\ahpair}$, as measured by the clipped version of $\Lcal$, is defined to be  $\llbracket \Lcal(y\alb\cdot\hb(x))\rrbracket_1$ where  $\llbracket x\rrbracket_1 \eqdef \min{}(x, 1)$. Finally, a function $\Acal: \Reals\ra\Reals$ is said to be $\ell$-Lipschitz for some $\ell > 0$ if and only if $|\Acal(x)-\Acal(x')| \le \ell|x - x'|$ for any $x$ and $x'$. 

\begin{theorem}
\label{thm:main}
Consider any distribution $D$ on $\Xcal\times\Ycal$. Consider any loss function $\Lcal$ which upper-bounds the zero-one loss and for which its clipped version is $\ell$-Lipschitz. Let $\Hcal$ be any symmetric class of real-valued functions on the input space $\Xcal$. For all $p\in[1,+\infty)$, for all $\dt\in (0,1]$, with probability at least $1-\dt$ over the random draws of $S\sim D^m$, we have simultaneously for all $\alb$ on $\Hcal$,
\begin{multline}\label{eq:main}
\esp{\ex\sim D} I(y \alb\cdot\hb(x) \le 0)\ \le\ \frac{1}{m}\sum_{i=1}^m \Lcal(y_i\alb\cdot\hb(x_i)) + 4\ell\dim(\alb)^{1-\frac{1}{p}}\|\alb\|_p \Rcal_m(\Hcal)\\
+ \sqrt{\frac{1}{2m}\log\frac{\pi^2(\dim(\alb)+1)^2}{6\dt}} +  \sqrt{\frac{1}{m}\log\log_2\LB 2 \|\alb\|_p\RB}\, .
\end{multline}
\end{theorem}
As it is usual with Rademacher complexities, Theorem~\ref{thm:main} also applies with $\Rcal_m(\Hcal)$ replaced by $\Rcal_S(\Hcal)$ if $\delta$ is replaced by $\delta/2$ and if the last term is multiplied by $3$.
\begin{proof}
The fundamental theorem on Rademacher complexities (see, for example, \citet{mrt-12,sc-04}) states that for any 
class $\Gcal$ mapping some domain $\Zcal$ to $[0,1]$, for any distribution $D$ on $\Zcal$, for any $\delta > 0$, with probability at least $1-\delta$, we have simultaneously for all $g\in\Gcal$
\[
\esp{z\sim D} g(z)\ \le\ \frac{1}{m}\sum_{i=1}^m g(z_i) + 2 \Rcal_m(\Gcal) + \sqrt{\frac{1}{2m}\log\frac{1}{\delta}}\, .
\]
Given any $\Lcal$, any $p\in[1,+\infty]$,  and any $\gm > 0$, we can apply this theorem to the set $\Gcal_\gm$ of functions that maps each example $(x,y)$ to $\llbracket \Lcal(y\frac{\qb}{\gm}\cdot\hb(x))\rrbracket_1$ for $\|\qb\|_p = 1$ when each $h_i\in\Hcal$.  
By hypothesis, the clipped version of $\Lcal$ is $\ell$-Lipschitz. Hence, by Talagran's lemma (see, for example, Theorem 4.2 of~\citet{mrt-12}), we have that $\Rcal_S(\Gcal_\gm) = \ell\Rcal_S(\Gcal'_\gm)$, where $\Gcal'_\gm$ is the set of functions mapping $(x,y)$ to $y\frac{\qb}{\gm}\cdot\hb(x)$. 

Also, since $\gm$ is a constant and $y=\pm1$, we have that $\Rcal_S(\Gcal'_\gm) = (1/\gm)\Rcal_S(\Ccal_p^n)$ where $\Ccal_p^n$ is the set of functions mapping $x$ to $\qb\cdot\hb(x)$ such that $\|\qb\|_p = 1$ and $\dim(\qb) = n$. Thus, $\Rcal_S(\Gcal_\gm) = (\ell/\gm)\Rcal_S(\Ccal_p^n) = (\ell/\gm)\dim(\qb)^{1-(1/p)}\Rcal_S(\Hcal)$ according to Lemma~\ref{lem:holder}. 

Then, for any $\gm > 0$, with probability at least $1-\dt$ we have 
\begin{multline*}
\esp{(x,y)\sim D} \llbracket \Lcal(y\frac{\qb}{\gm}\cdot\hb(x))\rrbracket_1\ \le\ \frac{1}{m}\sum_{i=1}^m \llbracket \Lcal(y_i\frac{\qb}{\gm}\cdot\hb(x_i))\rrbracket_1 + 2  (\ell/\gm)\dim(\qb)^{1-(1/p)}\Rcal_m(\Hcal) + \sqrt{\frac{1}{2m}\log\frac{1}{\delta}}\, .
\end{multline*}
By using the union bound technique of Theorem 4.5 of~\citet{mrt-12}, we can make the above bound valid uniformly over all values for $\gamma$ by adding $\sqrt{(1/m) \log\log_2(2/\gm)}$ to its right hand side and by multiplying the $\Rcal_m(\Hcal)$ term by 2. 
To obtain a bound which is also valid uniformly for all values of $n=\dim(\qb)$, we replace $\dt$ by $(6\dt/\pi^2)(n+1)^{-2}$.
Then, by using $\alb=\qb/\gm$, we have that $\|\alb\|_p = 1/\gm$. The theorem then follows from the fact that $I(y\alb\cdot\hb(x) \le 0)\le  \llbracket \Lcal(y\alb\cdot\hb(x))\rrbracket_1\le \Lcal(y\alb\cdot\hb(x))\ \forall (x,y)\in\Xcal\times\Ycal$.  
\end{proof}
If we ignore the slowly increasing logarithm terms in Equation~\eqref{eq:main}, Theorem~\ref{thm:main} tells us that to obtain a majority vote with a small zero-one loss, it is sufficient to minimize the empirical risk, as measured with respect to a surrogate loss $\Lcal$, plus a regularization term equal to $4\ell\dim(\alb)^{1-\frac{1}{p}}\|\alb\|_p \Rcal_m(\Hcal)$. Note that when $p=1$, this regularization term is equal to $4\ell\|\alb\|_1 \Rcal_m(\Hcal)$ and, therefore, does not depend on the number $\dim(\alb)$ of voters used by the majority vote. Hence, when performing $L_1$ regularization with a fixed set $\Hcal$ of voters  and surrogate loss $\Lcal$, the only thing that matters is to control the $L_1$ norm of $\alb$ while minimizing the empirical risk. This is in sharp contrast with the $p>1$ cases where we need to control both the $L_p$ norm of $\alb$ \emph{and} the number $\dim(\alb)$ of voters used by $f_\ahpair$. Therefore, iterative learning algorithms that minimize the empirical loss under $L_p$ regularization should also perform early stopping or exercise some other explicit control on the number of voters used by the majority vote when $p>1$. This explicit control on $\dim(\alb)$ is mostly important with $L_\infty$ regularization because the regularization term then grows linearly with $\dim(\alb)$. But it is also important with $L_2$ regularization since the regularization term then grows with $\sqrt{\dim(\alb)}$. Finally, and perhaps most importantly, just minimizing iteratively the empirical risk and using early stopping to control overfitting is a simple learning strategy that is supported by Theorem~\ref{thm:main}. Indeed, as long as the iterative procedure does not choose large weights for the voters, early stopping keeps $\|\alb\|_1$ under control and the right hand side (r.h.s.) of Equation~\eqref{eq:main} should be small when $\|\alb\|_1 \ll \sqrt{m}$ (since $\Rcal_m(\Hcal)\in O(1/\sqrt{m})$ when $\Hcal$ has finite VC dimension) and the empirical risk of the ensemble has reached a small value. 

The other important issue regarding Theorem~\ref{thm:main} concerns the choice of the surrogate loss $\Lcal$. Obviously, the closer (or tighter) $\Lcal$ is to the zero-one loss, the better. However, to avoid computational problems associated with the existence of several local minima of the empirical risk, let us settle for a surrogate $\Lcal$ \emph{convex} in $\alb$. One of the tightest convex surrogate that we can think of is the hinge loss. However, the presence of a discontinuity in its first derivative does not give rise to a simple boosting-type iterative algorithm, sometimes called \emph{forward stagewise additive modelling}~\citep{htf-09},  that attempt to produce a majority vote by iteratively adding voters chosen from a possibly continuous set $\Hcal$. The hinge loss does give rise to a linear programming algorithm, called LPBoost~\citep{dbs-02}, when used in conjunction with $L_1$ regularization. Although this learning strategy is strongly supported by Theorem~\ref{thm:main}, solving iteratively a linear program each time a new voter is inserted into the ensemble is a computationally expensive strategy when the number of voters in the ensemble is large. As machine learning is entering into the Big Data era, we want the restrict ourselves to forward stagewise additive modelling algorithms.  One solution is to use a smoother surrogate, containing no discontinuities in its derivatives, at the price of sacrificing a bit of the tightness with respect to the zero-one loss. The exponential loss minimized by AdaBoost~\citep{fs-97} has  continuous first derivatives but is very far from being a tight upper bound of the zero-one loss. The logistic loss minimized by LogitBoost~\citep{fht-00} is much better in that respect but does not produce simple update rules in the sense that each example of the training set needs to be reweighed each time a new voter is added to the ensemble. 
Are there simpler updates rules that perform as well as AdaBoost and LogitBoost? Let's try to answer this question by analyzing what happens if we use the very simple \emph{quadratic loss} for the surrogate $\Lcal$. The quadratic loss is commonly used in classical learning methods such as (kernel) ridge regression and back-propagation neural network learning and has been considered by~\citet{htf-09} for forward stagewise additive modelling algorithms. 
But these authors do not recommend the use of the quadratic loss as a surrogate for the zero-one loss and propose, instead, the use of more robust losses against outliers such as the squared hinge loss. However, such ``Huberized" losses do not give rise to simple boosting algorithms such as those obtained with the square loss below. The square loss has been considered for Boosting by~\cite{by-03} (under the name ``Boosting with the $L_2$ loss") and have obtained excellent performances when using cubic splines. However, they have not considered any $L_p$-regularized variant of Boosting. More recently, \citet{gllms-09} considered boosting with the quadratic loss with a Kullback-Leibler regularizer. Consequently, their boosting algorithm turned out to be different than the algorithms proposed in the next section. Moreover, their PAC-Bayesian theory based on quasi-uniform posteriors was developed only for the case of a finite set of voters and does not extend to the continuous case. Hence, it was not realized that it was necessary to control the number of voters when boosting with a $L_p$ regularizer with $p>1$, while no such control is needed for $p=1$.

\section{QuadBoost}

We now investigate if the quadratic loss can yield simple and efficient iterative algorithms for producing ensemble of voters. For this task, consider any $n\in\Naturals$, any vector $\hb\eqdef(h_1,\ldots,h_n)$ of voters where each $h_i\in\Hcal$, and any vector $\alb\eqdef(\al_1,\ldots,\al_n)$ of real-valued weights on $\hb$. Let us start by writing the quadratic risk (on $m$ examples) as 
\begin{multline}\label{eq:Qempirical}
\frac{1}{m} \sum_{k=1}^m (y_k - \alb\cdot\hb(x_k))^2\ = 1 - 2 \sum_{j=1}^n \al_j \frac{1}{m}\sum_{k=1}^m y_k h_j(x_k) 
+ \sum_{j=1}^n \al_j^2 \frac{1}{m}\sum_{k=1}^m h_j^2(x_k)\\ + 2 \sum_{j=2}^n \al_j \frac{1}{m}\sum_{k=1}^m h_j(x_k) \sum_{i=1}^{j-1} \al_i h_i(x_k)\, .
\end{multline}
If, for each voter $h_j$ of the ensemble, we now define its margin $\mu_j$ as 
\begin{equation}\label{eq:mu}
\mu_j\eqdef \frac{1}{m}\sum_{k=1}^m y_k h_j(x_k)\, ,
\end{equation}
and its \emph{correlation} $M_j$ \emph{with the weighted sum of the previous voters} as
\begin{equation}\label{eq:M}
M_j \eqdef\LC \begin{array}{ll}
                           \frac{1}{m}\sum_{k=1}^m h_j(x_k) \sum_{i=1}^{j-1} \al_i h_i(x_k) & \textrm{if}\ j > 1\\[1mm]
                           0  & \textrm{if}\ j = 1\, ,
                           \end{array}
                           \right.
\end{equation} 
we obtain the following decomposition of the quadratic risk
\begin{equation}\label{eq:Qdecomposition}
\frac{1}{m} \sum_{k=1}^m (y_k - \alb\cdot\hb(x_k))^2\ =\ 1- 2 \sum_{j=1}^n \al_j(\mu_j - M_j) + \sum_{j=1}^n \al_j^2\eta_j\, ,\quad\textrm{where}\ \eta_j\eqdef \frac{1}{m}\sum_{k=1}^m h_j^2(x_k) \, .
\end{equation}
This decomposition tells us that to minimize the quadratic risk iteratively, we should, at each step $j$, find a voter $h_j\in\Hcal$ that maximizes $|\mu_j - M_j|$.
Once a voter $h_j$ is chosen, its weight $\al_j$ that minimizes the quadratic risk is obtained 
by setting to $0$ its partial derivative with respect to $\alpha_j$. This gives
\begin{equation}\label{eq:al-empirical}
\al_j\ =\ \frac{1}{\eta_j} (\mu_j - M_j)\quad ;\quad\textrm{(without regularization)}\, .
\end{equation}
Note that the sign of $\al_j$ is given by the sign of $(\mu_j- M_j)$. Also, it is easy to verify that adding to the ensemble a voter $h_j$ with weight $\al_j$ given by Equation~\eqref{eq:al-empirical} \emph{decreases} the empirical quadratic risk of the ensemble by $(\mu_j-M_j)^2/\eta_j$. 
When it is computationally very expensive to find $h_j\in\Hcal$ maximizing $|\mu_j-M_j|$, we can settle to find more rapidly any $h_j$ having $|\mu_j-M_j| > 0$ since, in that case, we still make progress by lowering the empirical quadratic risk of the ensemble by $(\mu_j-M_j)^2/\eta_j$. 

\paragraph{Vanilla QuadBoost:}Inserting into the ensemble, at each step $j$, a voter $h_j\in\Hcal$ with the weight $\al_j$ given by Equation~\eqref{eq:al-empirical}, defines, what we call, the \emph{vanilla} version of QuadBoost. Of course, in that case, Theorem~\ref{thm:main} tells that we should eventually early stop this greedy process to avoid over fitting--which is also the case for AdaBoost. 

One reason often invoked for using AdaBoost is its exponentially fast decrease of the empirical error as a function of the number of iterations (boosting rounds). More precisely, assuming that, at each iteration, the weak learner can always produce a classifier achieving a training error (on the weighted examples) of at most $(1/2) - \gm$, the (zero-one loss) training error produced by the AdaBoost ensemble is at most $\exp(-2\gm^2 T)$ after $T$ iterations~\citep{fs-97}. Consequently, the number of iterations needed for AdaBoost to obtain an ensemble achieving less than $\epsilon$ empirical error is $[1/(2\gamma^2)]\log(1/\ep)$.  

In comparison, under the equivalent assumption that the weak learner is always able to find a voter $h_j\in\Hcal$ where $|\mu_j - M_j| > \gm$, the decrease in the quadratic empirical risk (which upper-bounds the zero-one training error) achieved by the QuadBoost ensemble is at least $(\mu_j - M_j)^2$ (since $\eta_j = 1$ for classifiers) at each iteration. Hence, under this hypothesis, the training error produced by the QuadBoost ensemble after $T$ iterations is at most $1-T\gm^2$. Hence, under this hypothesis, QuadBoost needs at most $(1/\gm^2)$ iterations to have an an ensemble achieving at most $\epsilon$ training error. Consequently, in comparison with AdaBoost, and under an equivalent hypothesis, the convergence rate of QuadBoost is slightly better. 

Let us now investigate the different $L_p$ regularized versions of QuadBoost for $p=1, 2$, and $+\infty$, in accordance with the insights given by Theorem~\ref{thm:main}. To this end, we first note that the clipped quadratic loss is $2$-Lipschitz, so $\ell=2$ in Theorem~\ref{thm:main} when $\Lcal$ is the quadratic loss. 

\paragraph{QuadBoost-$L_1$:} For $p=1$, we should minimize the empirical quadratic risk plus $2\ld\|\alb\|_1$, where, according to Theorem~\ref{thm:main}, $\ld$ should be equal to $4\Rcal_m(\Hcal)$ But, in practice, a smaller value for $\ld$ should provide better results as there are always some looseness in risk bounds. If we add this regularization term to the expression of the empirical risk given by Equation~\eqref{eq:Qdecomposition} and then set to zero the first derivative w.r.t. $\al_j$ of this objective, we find that, at each step $j$, the solution for $\al_j$ is given by
\begin{equation}\label{eq:al-L1}
\al_j\ =\ \LC
           \begin{array}{lll}
           \frac{1}{\eta_j} (\mu_j - M_j -\ld )& \textrm{IF} & \mu_j - M_j > \ld\\[1mm]
           - \frac{1}{\eta_j} (M_j - \mu_j -\ld )& \textrm{IF} & M_j - \mu_j > \ld\\[1mm]
           0 &\textrm{IF} & |\mu_j - M_j|\le \ld\quad (L_1\ \textrm{regularization})\, .
           \end{array}
           \right.
\end{equation}
It can be verified that this update rule gives a decrease of $(|\mu_j-M_j|-\ld)^2/\eta_j$ in the risk bound value of Theorem~\ref{thm:main} when $|\mu_j-M_j|>\ld$.  Here, no explicit early stopping is needed as, for some chosen $\ld > 0$, we will eventually be unable to find a voter $h_j$ having $|\mu_j-M_j| > \ld$. Hence, the amount of voters contained in the ensemble is controlled by parameter $\ld$: the larger $\ld$ is, the smaller the ensemble will be. Finally note that this algorithm can be viewed as an iterative version of the LASSO method~\citep{t-96,a-08} but where the functions are selected from a possibly continuous set $\Hcal$.  

\paragraph{QuadBoost-$L_2$:}For $p=2$, we should minimize the empirical quadratic risk plus $\ld\|\alb\|_2$, where, according to Theorem~\ref{thm:main}, $\ld$ should be equal to $8\sqrt{\dim(\alb)}\Rcal_m(\Hcal)$. If we add this regularization term to the expression of the empirical risk given by Equation~\eqref{eq:Qdecomposition} and then set to zero the first derivative w.r.t. $\al_j$ of this objective, we find that, at each step $j$, the solution for $\al_j$ is given by
\begin{equation}\label{eq:al-L2}
\al_j\ =\ \frac{1}{\eta_j + \ld} (\mu_j - M_j)\quad ; \quad (L_2\ \textrm{regularization})\, .
\end{equation}
It can be verified that this update rule gives a decrease of $(\mu_j-M_j)^2/(\eta_j + \ld)$ in the risk bound value. Since according to theory, $\ld$ should increase with the number $\dim(\alb)$ of voters in the ensemble, explicit early-stopping should be performed in addition to the above rule for $\al_j$. Finally note that this algorithm can be viewed as an iterative version of ridge regression but where the functions are selected from a possibly continuous set $\Hcal$.

\paragraph{QuadBoost-$L_\infty$:} For $p=+\infty$, we should minimize the empirical quadratic risk plus $\ld\|\alb\|_\infty$, where, according to Theorem~\ref{thm:main}, $\ld$ should be equal to $8\dim(\alb)\Rcal_m(\Hcal)$. Since $\|\alb\|_\infty = \al_{max}$, which is the allowed upper-bound for the weight values of the voters, each $\al_j$ should minimize the empirical risk provided that its absolute value does not exceed $\al_{max}$. The solution for $\al_j$ is then given by
\begin{equation}\label{eq:al-Linf}
\al_j\ =\ \LC \begin{array}{ll}
                \frac{1}{\eta_j} (\mu_j - M_j) & \textrm{IF}\quad \frac{1}{\eta_j}|\mu_j - M_j|\ \le\ \al_{max}\\[2mm]
                \al_{max}\, \sgn(\mu_j - M_j) & \textrm{IF}\quad \frac{1}{\eta_j}|\mu_j - M_j|\ >\ \al_{max};\quad (L_\infty\ \textrm{regularization})\, .
              \end{array}
              \right.
\end{equation}
Since according to theory, $\ld$ should increase linearly with the number $\dim(\alb)$ of voters in the ensemble, explicit early-stopping should be performed in addition to the above rule for $\al_j$. 

\subsection{Re-weighting the voters in the ensemble}

Up to now, we have been assuming that each voter added to the ensemble is new and different from all the others already present. Even though this should generally be the case when the class $\Hcal$ is infinitely large, the risk bound of Theorem~\ref{thm:main} tells us that it might be advantageous to occasionally re-weight the voters already present in the ensemble. Indeed, in that way, we do not change $\dim(\alb)$ in this process and only change $\|\alb\|_p$ so as to decrease the risk bound. Hence, let us investigate the update rules for the voter's weights in that case. 

Let us consider that we have $n$ voters currently in the ensemble and we want to change the weight $\al_j$ of voter $h_j$. Let us denote by $\alb'$ the weight vector formed by the weights of all voters in the ensemble except voter $h_j$. The quadratic empirical risk can now be written as
\begin{eqnarray}\label{eq:newQR}
    \nonumber \frac{1}{m} \sum_{k=1}^m (y_k-\alb\cdot\hb(x_k))^2 &=& \frac{1}{m} \sum_{k=1}^m (y_k-\alb'\cdot\hb(x_k)-\al_jh_j(x_k))^2\\
  \nonumber  &=& c - 2 \al_j\frac{1}{m}\sum_{k=1}^m h_j(x_k) (y_k - \alb'\cdot\hb(x_k)) + \al_j^2\sum_{k=1}^m h_j^2(x_k)\\
  &\eqdef& c - 2\al_j(\mu_j-M_j) + \al_j^2\eta_j\, ,
\end{eqnarray}
where $c$ is a term that does not depend on $\al_j$, $\mu_j$ and $\eta_j$ retain the same definitions as before, and $M_j$ is now defined as the correlation of voter $h_j$ with $\alb'\cdot\hb$.

\paragraph{Vanilla QuadBoost:} In this case, we set $\al_j$ so as to minimize only the quadratic risk above. This gives the same uptate rule as before; namely, $\al_j = (\mu_j-M_j)/\eta_j$. Since, this could increase $\|\alb\|_p$, Theorem~\ref{thm:main} tells us, that we should eventually stop updating the weights to avoid overfitting.

\paragraph{QuadBoost $L_p$:} For these cases, adding the $\al_j$-dependent regularization term, gives the same update rules as before except that $M_j$ is now defined as the correlation of voter $h_j$ with $\alb'\cdot\hb$. For the $L_1$ case, this means that a weight could eventually be set to zero, then yielding a sparser solution.

\section{Experimental Results}

Let us first note that, although all the boosting algorithms tested in this section can select the voters from a continuous set, we have, for the sake of comparison, used only finite sets. We feel that continuous sets of voters raise several important issues, including a non trivial trade off between precision, time complexity, and capacity (or Rademacher complexity), which clearly need extra work to be lucidly addressed and, consequently, should be treated in a separate paper.


We now report empirical experiments on binary classification datasets from the UCI Machine Learning Repository~\cite{uci-98}. Each dataset was randomly split into a training set $S$ of at least half of the examples and at most $500$ examples, and a testing set containing the remaining examples. Each dataset has been normalized using a hyperbolic tangent, whose parameters have been chosen using the training set $S$ only.
We considered a finite set of \emph{decision stumps} (one-level decision trees) which consists of a single input attribute and a threshold. For each dataset, we generated $10$ decision stumps per attribute and their complements. 

We compared QuadBoost  ($L_1$, $L_2$, $L_\infty$, and its vanilla version that has no regularizer) with LPBoost~\cite{dbs-02}, and AdaBoost~\cite{fs-97}. For each algorithm, all hyperparameters have been chosen amont $10$ values in a logarithmic scale, by performing $5$-fold cross-validation on the training set $S$, and the reported values are the risks on the testing set. 

For QuadBoost-$L_1$, $\lambda$ was chosen in a range between $10^{-4}$ and $10^0$. For QuadBoost-$L_2$, the range for $\lambda$ was between $10^0$ and $10^3$, and the range for the number of iteration $T$ was between $10^1$ and $10^5$. For QuadBoost-$L_\infty$, $\alpha_{max}$ was chosen between $10^{-4}$ and $10^{-1}$, and $T$ between $10^0$ and $10^5$. Vanilla QuadBoost's $T$ was chosen between $10^0$ and $10^3$, hyperparameter $C$ of LPBoost was chosen between $10^{-3}$ and $10^3$, and finally the number of iterations $T$ of AdaBoost was chosen between $10^2$ and $10^6$.

\begin{table*}[h]
	
\begin{center}
\begin{scriptsize}
\rowcolors{4}{black!10}{}
\begin{tabular}{lcccccc}
\toprule
 Dataset & QuadBoost-$L_1$ & QuadBoost-$L_2$ & QuadBoost-$L_\infty$ & QuadBoost & LPBoost & AdaBoost \\
\cmidrule(l r){1-1} \cmidrule(l r){2-2} \cmidrule(l r){3-3} \cmidrule(l r){4-4} \cmidrule(l r){5-5} \cmidrule(l r){6-6} \cmidrule(l r){7-7}
australian & \textbf{0.145} & \textbf{0.145} & \textbf{0.145} & \textbf{0.145} & \textbf{0.145} & 0.191 \\
balance & 0.035 & \textbf{0.022} & 0.029 & 0.026 & 0.029 & 0.035 \\
breast & 0.054 & 0.049 & 0.046 & 0.046 & \textbf{0.043} & 0.046 \\
bupa & \textbf{0.279} & 0.285 & 0.308 & 0.297 & 0.349 & 0.372 \\
car & 0.164 & 0.160 & 0.150 & 0.153 & 0.155 & \textbf{0.130 }\\
cmc & 0.300 & \textbf{0.292} & 0.296 & 0.306 & 0.311 & 0.310 \\
credit & 0.133 & 0.133 & 0.133 & 0.133 & \textbf{0.128} & 0.165 \\
cylinder & 0.285 & 0.311 & \textbf{0.270} & 0.278 & 0.278 & 0.281 \\
ecoli & \textbf{0.065} & \textbf{0.065} & 0.083 & 0.089 & 0.113 & 0.095 \\
flags & 0.309 & 0.278 & 0.309 & 0.268 & 0.299 & \textbf{0.247 }\\
glass & 0.159 & \textbf{0.140} & 0.150 & 0.150 & 0.290 & 0.215 \\
heart & 0.200 & \textbf{0.178} & 0.200 & 0.215 & 0.230 & 0.230 \\
hepatitis & 0.221 & 0.221 & 0.221 & 0.221 & 0.221 & \textbf{0.182 }\\
horse & 0.207 & 0.163 & \textbf{0.158} & 0.207 & 0.207 & 0.185 \\
ionosphere & \textbf{0.091} & 0.126 & \textbf{0.091} & 0.120 & 0.120 & 0.114 \\
letter\_ab & 0.010 & \textbf{0.006} & \textbf{0.006} & \textbf{0.006} & 0.011 & 0.010 \\
monks & \textbf{0.231} & \textbf{0.231} & \textbf{0.231} & \textbf{0.231} & \textbf{0.231} & 0.255 \\
optdigits & 0.086 & 0.077 & \textbf{0.074} & 0.078 & 0.096 & 0.081 \\
pima & 0.258 & \textbf{0.237} & 0.245 & 0.268 & 0.253 & 0.273 \\
tictactoe & 0.315 & \textbf{0.313} & \textbf{0.313} & 0.322 & 0.342 & 0.317 \\
titanic & 0.228 & 0.228 & 0.228 & 0.228 & \textbf{0.222} & \textbf{0.222 }\\
vote & \textbf{0.051} & \textbf{0.051} & \textbf{0.051} & \textbf{0.051} & \textbf{0.051} & \textbf{0.051 }\\
wine & 0.079 & 0.056 & 0.056 & 0.090 & 0.067 & \textbf{0.045 }\\
yeast & \textbf{0.283} & 0.290 & 0.296 & 0.296 & 0.300 & 0.300 \\
zoo & \textbf{0.040} & 0.100 & \textbf{0.040} & \textbf{0.040} & 0.120 & 0.120 \\
\midrule
Mean running time (seconds) & 1.326 & 74.040 & 1.131 & 0.397 & 26.930 & 8.096\\
\bottomrule
\end{tabular}
\end{scriptsize}
\end{center}
	\caption{Testing risks of four versions of QuadBoost, compared with LPBoost and AdaBoost. The bold value corresponds to the lowest testing risk among all algorithms. The last line reports the mean running time of the algorithms for all datasets.}
	\label{tab:results_stumps}
\end{table*}

Table~\ref{tab:results_stumps} reports the resulting testing risks and training times. 
The results show that all four variants of QuadBoost that we considered are competitive with state-of-the-art boosting algorithms. When comparing vanilla QuadBoost with AdaBoost, where the only hyperparameter to tune is the number of iterations, QuadBoost wins or ties $17$ times over $25$ datasets and is $20$ times faster. When comparing QuadBoost-$L_1$ with LPBoost, which is also a $L_1$-norm regularized algorithm, QuadBoost outperforms or ties with LPBoost $16$ times over $25$ datasets and is also $20$ times faster.

Table~\ref{tab:poisson_stumps} shows a statistical comparison between all these algorithms, using the pairwise Poisson binomial test of~\citet{lacoste-2012}. Given a set of datasets, this test gives the probability that a learning algorithm is better than another one. This table also shows the pairs of algorithms having a significant performance difference using the pairwise sign test~\cite{mendenhall1983nonparametric}. The only significant values indicate that QuadBoost with $L_\infty$-norm and $L_2$-norm regularization outperform QuadBoost without regularization, and that QuadBoost with $L_2$-norm regularization outperforms QuadBoost with $L_1$-norm regularization. Note however that for the two former algorithms, we have performed cross-validation over two hyperparameters, which gave them an advantage.  Another possible explanation for the observed improved performance of the $L_2$ and $L_\infty$ regularized versions is the increased Rademacher complexity of $L_2$ and $L_\infty$ combinations over $L_1$ combinations.

\begin{table}[h]
\rowcolors{2}{black!10}{}
\begin{center}
\begin{small}
\begin{tabular}{rcccccc}
\toprule
          \textbf{\textbf{}} & QuadB.-$L_\infty$  & QuadB.-$L_2$ & LPBoost & AdaBoost & QuadB. & QuadB.-$L_1$ \\
\cmidrule(l r){2-2} \cmidrule(l r){3-3} \cmidrule(l r){4-4} \cmidrule(l r){5-5} \cmidrule(l r){6-6} \cmidrule(l r){7-7}
QuadBoost-$L_\infty$ & \gray{0.50}$\phantom{^\star}$ & 0.48$\phantom{^\star}$ & 0.49$\phantom{^\star}$ & 0.55$\phantom{^\star}$ & 0.80$^\star$ & 0.81$^\star$ \\
QuadBoost-$L_2$      & \gray{0.52}$\phantom{^\star}$ & \gray{0.50}$\phantom{^\star}$ & 0.44$\phantom{^\star}$ & 0.45$\phantom{^\star}$ & 0.75$\phantom{^\star}$ & 0.84$^\star$ \\
LPBoost              & \gray{0.51}$\phantom{^\star}$ & \gray{0.56}$\phantom{^\star}$ & \gray{0.50}$\phantom{^\star}$ & 0.42$\phantom{^\star}$ & 0.69$\phantom{^\star}$ & 0.67$\phantom{^\star}$ \\
AdaBoost             & \gray{0.45}$\phantom{^\star}$ & \gray{0.55}$\phantom{^\star}$ & \gray{0.58}$\phantom{^\star}$ & \gray{0.50}$\phantom{^\star}$ & 0.58$\phantom{^\star}$ & 0.57$\phantom{^\star}$ \\
QuadBoost            & \gray{0.20$^\star$} & \gray{0.25}$\phantom{^\star}$ & \gray{0.31}$\phantom{^\star}$ & \gray{0.42}$\phantom{^\star}$ & \gray{0.50}$\phantom{^\star}$ & 0.48$\phantom{^\star}$ \\
QuadBoost-$L_1$      & \gray{0.19$^\star$} & \gray{0.16$^\star$} & \gray{0.33}$\phantom{^\star}$ & \gray{0.43}$\phantom{^\star}$ & \gray{0.52}$\phantom{^\star}$ & \gray{0.50}$\phantom{^\star}$ \\
\bottomrule
\end{tabular}
\end{small}
\caption{Pairwise Poisson binomial test between all pairs of algorithms. A gray value indicates redundant information, and a star indicates that the difference between the two algorithms is also significant using the pairwise sign test, with a $p$-value of $0.05$.}
\label{tab:poisson_stumps}
\end{center}
\end{table}

In conclusion, the empirical experiments show that QuadBoost is a fast and accurate ensemble method that competes well against other state of the art boosting algorithms.


\section{Conclusion}
\label{sect:conclu}
We have presented a uniform risk bound for ensembles which holds for any surrogate loss and $L_p$ norm of the weighted combination of voters which can be selected from a continuous set. An important feature of this result is the fact that weighted combinations of unit $L_p$ norm for $p>1$ have strictly larger Rademacher complexity than weighted combinations of unit $L_1$ norm and, as a consequence, the risk bound exhibits an explicit dependence on the number of voters when $p>1$ while no such dependence occurs when $p=1$. This result suggests to perform an explicit control of the number of voters when regularizing with the $L_p$ norm for $p>1$ while no such control is needed for $p=1$. Finally, our theoretical and empirical results suggest that the simple quadratic loss surrogate should be used for boosting instead of the usual exponential loss.     

\clearpage
\bibliography{mario-bib}
\bibliographystyle{unsrtnat}



\end{document}